\documentclass[letterpaper, 10 pt, conference]{ieeeconf}  

\IEEEoverridecommandlockouts                              

\DeclareMathAlphabet{\mathcal}{OMS}{cmsy}{m}{n}
\DeclareSymbolFont{largesymbols}{OMX}{cmex}{m}{n}
\usepackage{graphicx}
\usepackage{bm}
\usepackage{graphics} 
\usepackage{epsfig} 
\usepackage{times} 
\usepackage{subfigure}
\usepackage{stfloats}
\usepackage{algorithm} 
\usepackage{algorithmic} 
\usepackage{multirow} 
\usepackage{booktabs}
\usepackage{amssymb}  
\usepackage{amsmath,amsfonts}
\newtheorem{theorem}{\bf Theorem}

\usepackage{fancyhdr}
\usepackage{pifont}
\usepackage{microtype}
\usepackage{booktabs}   
\usepackage{multirow}   
\usepackage{array}      
\usepackage{makecell}
\usepackage[hidelinks]{hyperref}

{} 

\author{
Yaohua Zhou$^{1}$, Mingdao Lin$^{2}$, Jiayu Song$^{3}$, Xixin Zhao$^{4}$, Haidong Huang$^{5}$, Jiahu Qin$^{6}$ and Xiaocong Li$^{7,*}$%
\thanks{*Corresponding Author.}
\thanks{$^{1}$Yaohua Zhou is with the College of Information Science and Technology, Eastern Institute of Technology, Ningbo 315200, China, and also with the Department of Automation, University of Science and Technology of China, Hefei 230027, China (e-mail: yhzhou@eitech.edu.cn).}%
\thanks{$^{2}$Mingdao Lin is with the College of Information Science and Technology, Eastern Institute of Technology, Ningbo 315200, China, and also with the Department of Electrical and Electronic Engineering, The Hong Kong Polytechnic University, Hong Kong SAR, China (e-mail: 25086082r@connect.polyu.hk).}%
\thanks{$^{3}$Jiayu Song is with the College of Information Science and Technology, Eastern Institute of Technology, Ningbo 315200, China (e-mail: jysong@eitech.edu.cn).}%
\thanks{$^{4}$Xixin Zhao is with the Department of Interdisciplinary Information Studies, The University of Tokyo, Japan (e-mail: 
zhaoxixin@g.ecc.u-tokyo.ac.jp).}%
\thanks{$^{5}$Haidong Huang is with the University of Nottingham Ningbo China, Ningbo 315100, China (e-mail: ssyhh10@nottingham.edu.cn).}%
\thanks{$^{6}$Jiahu Qin is with the Department of Automation, University of Science and Technology of China, Hefei 230027, China (e-mail: jhqin@ustc.edu.cn).}%
\thanks{$^{7}$Xiaocong Li is with the College of Information Science and Technology and Zhejiang Key Laboratory of Industrial Intelligence and Digital Twin, Eastern Institute of Technology, Ningbo 315200, China (e-mail: xiaocongli@eitech.edu.cn).}%
}

\begin{document}

\title{Joint-Angle Invariance of Jacobian-Based Performance Indices for Serial Robotic Manipulators}

\maketitle

\begin{abstract}
This article investigates the joint-angle invariance (JAI) of Jacobian-based performance indices for serial robotic manipulators. We identify kinematic conditions under which representative performance indices, including manipulability, stiffness, and other Jacobian-derived indices, remain independent of the first and last joint angles. This property allows a systematic reduction in the effective configuration-space dimensionality, thereby simplifying the performance evaluation. The utility of JAI is demonstrated through two application scenarios: workspace performance map construction and task layout optimization. Comparative results show that the proposed JAI-based strategy reduces computational effort and storage requirements without compromising the effectiveness.
\end{abstract}


\section{Introduction}

Serial robotic manipulators have been widely deployed in manufacturing, medical, and aerospace industries due to their large workspaces, structural simplicity, and versatility. The evaluation of the motion/force transmission capabilities relies heavily on Jacobian-based performance indices. The Jacobian captures the instantaneous mapping between joint velocities and Cartesian twists, as well as between joint torques and Cartesian wrenches. On this basis, performance indices are derived to measure various aspects of robot behavior. Notable examples include the manipulability index~\cite{yoshikawa1985manipulability}, which characterizes the motion capability, and the stiffness index~\cite{salisbury1980active}, which evaluates the magnitude of resistance to applied wrenches.

For decades, these performance indices have been applied in robotics. They serve as critical tools in robot design optimization, where geometric parameters are chosen to maximize dexterity or stiffness in targeted regions of the workspace. In trajectory planning, they are used to avoid singularities and maintain desirable kinematic capability. Redundancy resolution strategies frequently incorporate these indices to exploit extra degrees of freedom (DOFs) for enhancing task performance, such as maximizing manipulability. Furthermore, task layout optimization, such as selecting the placement pose of fixtures or workpieces, has been shown to benefit significantly from Jacobian-based performance analysis.

\subsection{Related Work}

The investigation of performance distribution within the robot workspace facilitates the evaluation of its capabilities. Breitkopf~\cite{breitkopf2023modelling} visualized the distribution of manipulability across the entire workspace of a serial robot. Kuang et al.~\cite{kuang2024research} employed a sector space meshing method to construct the distribution of the Jacobian condition number, considering the  workspace boundary. Such analyses provide valuable guidance for the design and optimization of robots. For instance, Ginnante et al.~\cite{ginnante2023task} utilized a task priority control algorithm to optimize task-oriented robot design, thereby enhancing the kinetostatic performance in desired regions of the workspace. In contrast, Hwang et al.~\cite{hwang2017design} focused on a comprehensive dimensional synthesis based on global performance indices.

Designing trajectory optimization schemes that incorporate Jacobian-based performance indices can effectively enhance the execution capability of robotic manipulators. For redundant robots, manipulability optimization has been widely recognized as a fundamental approach to improving motion capability while avoiding singular configurations. Tan et al.~\cite{tan2025manipulability} proposed a manipulability optimization method based on singular value analysis, which avoids the instability associated with pseudo-inverse computations. Yang et al.~\cite{yang2023kinematic} extended the optimization beyond the velocity level and developed a manipulability maximization scheme at the acceleration level. To improve computational efficiency in redundancy resolution, Shen et al.~\cite{shen2023concentric} introduced a derivative propagation approach. Further advancements include the optimization of velocity/force manipulability in bilateral teleoperation of coordinated remote redundant manipulators~\cite{yang2024passivity}. More recently, cross-domain manipulability transfer has gained attention. For instance, Gong et al.~\cite{gong2025manipulability} incorporated Riemannian manifold theory to achieve cross-domain mapping of human-robot manipulability ellipsoids via parallel transport. In manufacturing processes such as grinding, polishing, and milling, various stiffness-based trajectory optimization methods have been explored, with the aim of mitigating vibrations during execution~\cite{xiong2019stiffness,zhou2022sampling,zhang2024novel}.

In the field of robotic manipulation and manufacturing, task layout optimization generally refers to the systematic adjustment of the workpiece pose, so as to enhance overall operational performance. Vahrenkamp et al.~\cite{vahrenkamp2012manipulability} proposed an extended manipulability measure and applied it to online grasp selection for humanoid robots. For tasks involving the manipulation of large-scale objects, Franceschi et al.~\cite{franceschi2022optimal} introduced a hierarchical optimization framework ranging from global to local levels to optimize work cell layout. In specific process applications, Jiao et al.~\cite{jiao2018processing} focused on robotic drilling systems, developing a redundancy optimization model that considers directional stiffness performance. In machining optimization incorporating stiffness thresholds, Liao et al.~\cite{liao2021optimization} concentrated on robotic milling, optimizing both the robot redundancy and workpiece setup. Regarding autonomous exploration and planning, Liu et al.~\cite{liu2024manipulability} established an efficient exploration information system capable of continuously identifying high-manipulability configurations, thereby enabling adaptive layout decisions in unknown environments. Moreover, the adjustment of the robot base constitutes a critical aspect of task layout optimization, with typical applications including manipulator repositioning via mobile platforms for the machining of large components~\cite{wang2024robot}.

\subsection{Overview}

Jacobian-based performance indices are widely used in robotics, while their dependence on the full joint configuration presents two challenges. First, in high-DOF systems, these indices tend to vary considerably across the joint space, resulting in high computational and visualization costs. Second, the absence of identifiable invariance properties hinders theoretical analysis and limits their applicability in efficient optimization. To address these issues, it is essential to identify structural properties that simplify analysis without loss of generality. 

This study introduces joint-angle invariance (JAI)--a property exhibited by specific Jacobian-based performance indices for a class of serial robots, whereby the indices remain independent of the first and last joint angles. This invariance substantially reduces the computational burden in both performance analysis and optimization. The utility of JAI is demonstrated in two application scenarios: performance map construction and task layout optimization, and we achieve notable gains in computational efficiency in both cases.

\section{Jacobian-Based Performance Indices for Serial Robotic Manipulators}

This section presents a summary of Jacobian-based performance indices for serial robotic manipulators, focusing on kinematic and stiffness performance.

In robotics, the Jacobian describes the differential mapping between joint space and Cartesian space. Let $m$ denote the dimension of the Cartesian space and $n$ denote the joint DOF, then the above mapping can be expressed as 
\begin{equation}
\bm V=\bm J({\bm q}) \dot{\bm q},
\label{V=Jq}
\end{equation}
where $\bm V = [\bm \omega^{\rm T},\bm v^{\rm T}]^{\rm T} \in \mathbb{R}^{m \times 1}$ is the Cartesian-space velocity, and $\bm J(\bm q) \in \mathbb{R}^{m \times n}$, which is determined by the joint configuration $\bm q\in \mathbb{R}^{n \times 1}$, denotes the Jacobian. To facilitate the description of Cartesian performance, the representation based on the body velocity and body Jacobian is adopted in this study. Based on the force mapping and deformation relationship, the Cartesian stiffness matrix can be formulated as 
\begin{equation}
\bm K=\left(\bm {J}\bm K_q^{-1} \bm J ^{\rm T}\right)^{-1}
\end{equation}
where $\bm K_q$ is the joint stiffness matrix.

The kinematic and stiffness performance of a robot are commonly characterized by velocity transmission and deformation resistance, which are analyzed with respect to $\bm J$ and $\bm K$, respectively. Table~\ref{Jacobian-Based Performance Indices for Serial Manipulators} summarizes commonly used Jacobian-based performance indices, encompassing two categories: kinematic and stiffness performance. These indices are further classified into directional (line-based), planar (areal-based), and volumetric measures, together with condition number-based isotropy indices (2-norm and Frobenius-norm). All indices are derived from the robot's Cartesian manipulability or stiffness ellipsoid and possess corresponding geometric representations.

\begin{table*}[htbp]
        \centering
        \renewcommand{\arraystretch}{2}
        \caption{Jacobian-Based Performance Indices for Serial Manipulators.}
        \label{Jacobian-Based Performance Indices for Serial Manipulators}
        \begin{tabular}{m{2cm}<{\centering}m{4cm}<{\centering}m{5.5cm}<{\centering}m{4.5cm}<{\centering}}
        \toprule
        Category & Performance index & Mathematical Form & Description \\
        \midrule
        \multirow{6}{*}{\makecell[c]{Kinematic \\ performance}} & Directional manipulability index & $M_D=\dfrac{1}{\sqrt{\bm{\xi}^{\rm T}(\bm{JJ}^{\rm T})^{-1}\bm{\xi}}}$ & \makecell[c]{Manipulability along \\ a unit velocity direction $\bm \xi$}  \\
        & Planar manipulability index & $M_P=\dfrac{1}{\sqrt{\det(\bm{P}^{\rm T}(\bm{JJ}^{\rm T})^{-1}\bm{P})}}$ & \makecell[c]{Manipulability within a plane \\ defined by $\bm P \in \mathbb{R}^{m \times 2}$} \\
        & Volumetric manipulability index & $M_V=\dfrac{1}{\sqrt{\det(\bm{Q}^{\rm T}(\bm{JJ}^{\rm T})^{-1}\bm{Q})}}$ & \makecell[c]{Manipulability over a subspace \\ defined by $\bm Q \in \mathbb{R}^{m \times l}(3 \leqslant l \leqslant m)$} \\
        & \makecell[c]{Jacobian condition number index \\ (2-norm)} & $M_{C,2}=\dfrac{\sigma_{\min}(\bm J)}{\sigma_{\max}(\bm J)}$ & \makecell[c]{Kinematic isotropy: ratio of \\ extreme singular values} \\
        & \makecell[c]{Jacobian condition number index \\  (Frobenius-norm)} & $M_{C,F}=\dfrac{m}{\sqrt{{\rm trace} \left(\bm{JJ}^{\rm {T}}\right){\rm trace} \left((\bm{JJ}^{\rm {T}})^{-1}\right)}}$ & \makecell[c]{Normalized isotropy considering \\  all singular values} \\
        \midrule
        \multirow{6}{*}{\makecell[c]{Stiffness \\ performance}} & Directional stiffness index & $K_D=\dfrac{1}{\sqrt{\bm{\zeta}^{\rm T}\bm{K}^{-2}\bm{\zeta}}}$ & \makecell[c]{Stiffness along \\ a unit force direction $\bm \zeta$} \\
        & Planar stiffness index & $K_P=\dfrac{1}{\sqrt{\det(\bm{P}^{\rm T}\bm{K}^{-2}\bm{P})}}$ & \makecell[c]{Stiffness within a plane \\ defined by $\bm P \in \mathbb{R}^{m \times 2}$} \\
        & Volumetric stiffness index & $K_V=\dfrac{1}{\sqrt{\det(\bm{Q}^{\rm T}\bm{K}^{-2}\bm{Q})}}$ & \makecell[c]{Stiffness over a subspace \\ defined by $\bm Q\in \mathbb{R}^{m \times l}(3 \leqslant l \leqslant m)$} \\
        & \makecell[c]{Stiffness condition number index \\ (2-norm)} & $K_{C,2}=\dfrac{\lambda_{\min}(\bm K)}{\lambda_{\max}(\bm K)}$ & Measure of stiffness isotropy \\
        & \makecell[c]{Stiffness condition number index \\ (Frobenius-norm)} & $K_{C,F}=\dfrac{m}{\sqrt{{\rm trace} \left(\bm{K}^2\right){\rm trace} \left(\bm{K}^{-2}\right)}}$ & Normalized isotropy of stiffness \\
        \bottomrule
        \end{tabular}
\end{table*}

For a general serial robot, the minimum-norm inverse solution $\dot{\bm q}=\bm J^{\rm T} (\bm {JJ}^{\rm T})^{-1}\bm V$ is derived from~\eqref{V=Jq}. By defining a joint-space unit sphere $\dot{\bm q}^{\rm T}\dot{\bm q}=1$, we further derive the analytical equation of the manipulability ellipsoid:
\begin{equation}
\bm V^{\rm T}\left(\bm {JJ}^{\rm T}\right)^{-1}\bm V=1.
\label{manipulability ellipsoid}
\end{equation}
When a wrench $\bm F$ is applied to the robot's end-effector, its deformation is considered as $\bm D=\bm K^{-1}\bm F$. Correspondingly, by defining a unit sphere $\bm D^{\rm T}\bm D=1$, an analytical equation for the stiffness ellipsoid is obtained:
\begin{equation}
\bm F ^{\rm T}\bm K^{-2}\bm F=1.
\label{stiffness ellipsoid}
\end{equation}
Both manipulability ellipsoid and stiffness ellipsoid are $m$-dimensional, and the lengths of their semi-axes along each orthogonal direction correspond to the singular values of $\bm J$ and eigenvalues of $\bm K$, respectively. When $m=3$, the ellipsoid can be directly visualized, as shown in Fig.~\ref{Cartesian manipulability/stiffness ellipsoid}. Its center is located at the robot's end-effector, representing the manipulability or stiffness performance. As shown in Fig.~\ref{Cartesian manipulability/stiffness ellipsoid} and Table~\ref{Jacobian-Based Performance Indices for Serial Manipulators}:

1) $M_D$ and $K_D$: directional (linear) performance indices. Let $\bm V=M_D\bm \xi$, where $M_D=||\bm V||_2$ and $\bm \xi$ is a unit velocity direction. Substituting it into \eqref{manipulability ellipsoid}, we obtain $M_D^2 \bm \xi^{\rm T} \left(\bm {JJ}^{\rm T}\right)^{-1}\bm \xi=1$, thus deriving the analytical expression for $M_D$ (Table~\ref{Jacobian-Based Performance Indices for Serial Manipulators}). $M_D$ characterizes the velocity transmission capability along $\bm \xi$, under the condition of unit joint velocity ($\dot{\bm q}^{\rm T}\dot{\bm q}=1$). Similarly, let $\bm F=K_D\bm \zeta$, then $K_D^2\bm \zeta ^{\rm T}\bm K^{-2}\bm \zeta=1$, thus obtaining $K_D$. Stiffness is higher along $\bm \zeta$ if a larger wrench is required to achieve a unit deformation ($\bm D^{\rm T}\bm D=1$).

2) $M_P$, $M_V$, $K_P$ and $K_V$: planar and volumetric performance indices. Consider an $l$-dimensional subspace defined by $(m-l)$ linear constraints:
\begin{equation}
\bm C \bm V=\bm 0, \, \bm C \bm F=\bm 0
\label{homogeneous system}
\end{equation}
where $\bm C \in \mathbb{R}^{(m-l) \times m}$ is the constraint matrix, whose row vectors specify $m-l$ orthogonal hyperplane normals, confining $\bm V$ (or $\bm F$) to the null space of $\bm C$. The subspace is denoted by $\bm P$ when $l=2$, and by $\bm Q$ when $3 \leqslant l \leqslant m$. Both $\bm P$ and $\bm Q$ are selected as matrices with orthonormal columns. In practice, performance indices are typically evaluated with respect to the origin of the tool frame; hence, the hyperplanes are assumed to pass through the origin, i.e., \eqref{homogeneous system} is considered as a homogeneous system. Taking $M_P$ as an example, the subspace velocity $\bm V_{sub}$ satisfies $\bm V=\bm P \bm V_{sub}$. Substituting this into \eqref{manipulability ellipsoid}, we obtain $\bm V_{sub}^{\rm T} \bm P^{\rm T} \left(\bm {JJ}^{\rm T}\right)^{-1}\bm P \bm V_{sub}=1$. The area of this ellipse can be characterized by $1/\sqrt{\det(\bm{P}^{\rm T}(\bm{JJ}^{\rm T})^{-1}\bm{P})}$, describing the planar manipulability $M_P$. Similarly, $M_V$, $K_P$, and $K_V$ can be derived. In practical applications, $K_P$ is often taken as the square root, while $K_V$ is taken as the root of order ${\rm rank}(\bm Q)$, to achieve normalization and improve numerical stability. Two typical examples are provided below to illustrate the planar and volumetric performance indices.

(i) When $m=6$ and the hyperplane constraints are $\omega_x=\omega_y=\omega_z=v_z=0$ (or $\tau_x=\tau_y=\tau_z=f_z=0$), then
\begin{equation}
\bm C = \left[ {\begin{array}{*{20}{c}}
  1&0&0&0&0&0 \\ 
  0&1&0&0&0&0 \\ 
  0&0&1&0&0&0 \\ 
  0&0&0&0&0&1 
\end{array}} \right], \bm P = \left[ {\begin{array}{*{20}{c}}
  0&0 \\ 
  0&0 \\ 
  0&0 \\ 
  1&0 \\ 
  0&1 \\ 
  0&0 
\end{array}} \right].
\end{equation}
In this case, $M_P$ represents the manipulability within the $v_x$-$v_y$ plane, while $K_P$ comprehensively considers the linear stiffness in both $x$ and $y$ directions.

(ii) When $m=6$ and there are no constraints, then $\bm C = {\emptyset _{0 \times 6}}$ and $\bm Q = { \rm{\bm I}_{6 \times 6}}$. In this case, $M_V$ and $K_V$ represent the overall manipulability and stiffness performance of the robot, i.e., $M_V=\sqrt{\det(\bm {JJ}^{\rm T})}$ and $K_V=\det(\bm K)$. 

\begin{figure}[hptb]
        \centering
        \includegraphics[width=0.3\textwidth]{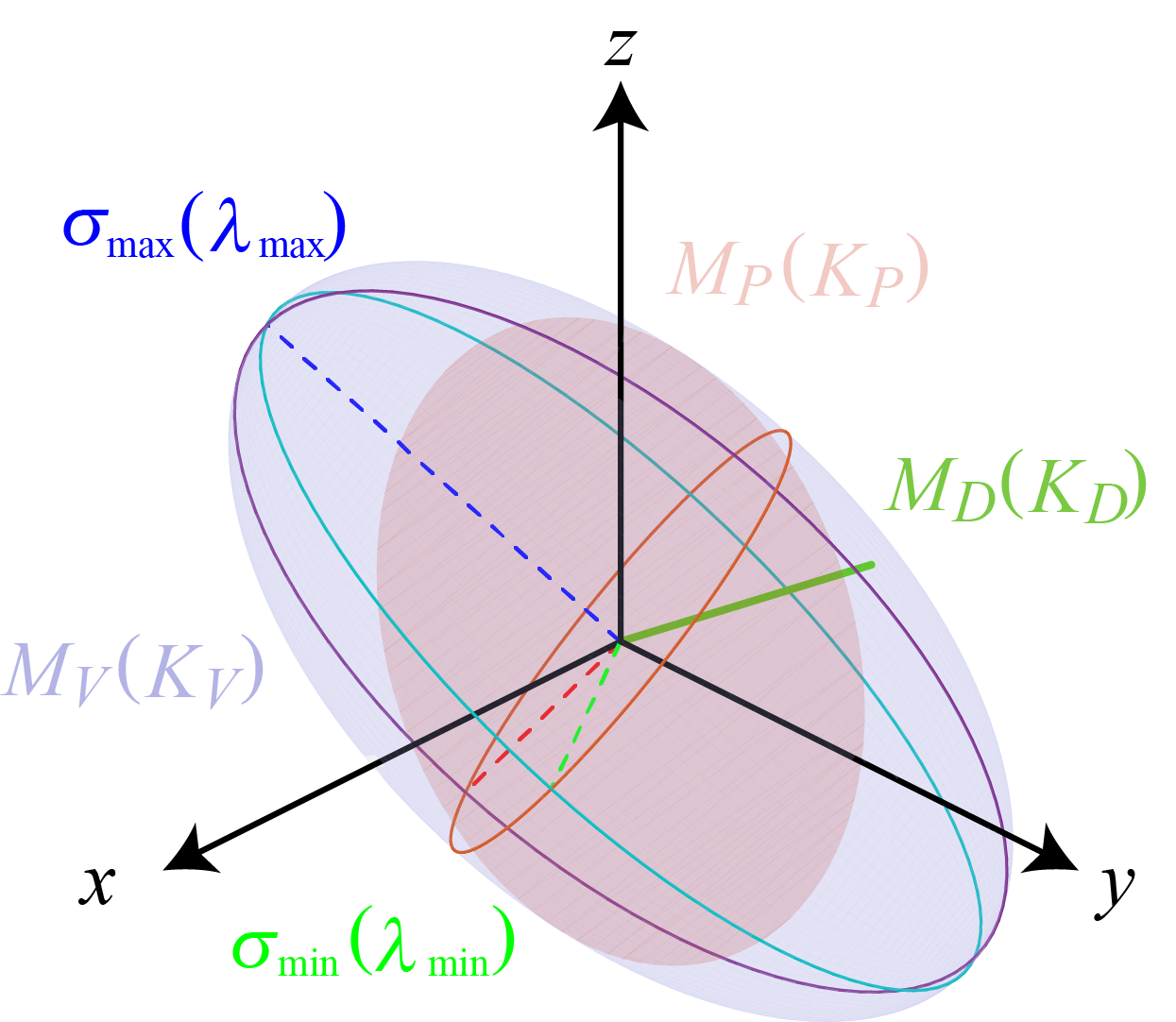}
        \caption{Cartesian manipulability/stiffness ellipsoid of the robot when $m=3$. Those Jacobian-based performance indices presented in Table~\ref{Jacobian-Based Performance Indices for Serial Manipulators} possess distinct geometric interpretations within this ellipsoid.}
        \label{Cartesian manipulability/stiffness ellipsoid}
\end{figure}

3) $M_{C,2}$, $M_{C,F}$, $K_{C,2}$ and $K_{C,F}$: condition number-based performance indices. These four indices are defined as reciprocals of the corresponding condition numbers, representing the dexterity and transmission accuracy of the robot. The 2-norm condition number is only sensitive to the extreme singular values (eigenvalues), while the Frobenius-norm condition number depends on all singular values (eigenvalues).

\section{JAI of Jacobian-based Performance Indices}

\subsection{Applicable Serial Robotic Manipulators}

The JAI discussed in this paper is primarily formulated for the serial robot configuration illustrated in Fig.~\ref{robot}. Many industrial manipulators (e.g., UR, KUKA, ABB) satisfy this structure or can be reparameterized accordingly. The base plane refers to the plane where the robot is mounted, while the tool plane denotes the plane perpendicular to the last link. The robot is required to satisfy the following mechanism characteristics: the first and last joints are revolute joints whose axes are perpendicular to the base plane and the tool plane, respectively. There are no specific requirements for intermediate joints. Frames are typically established by convention (not strictly mandatory) such that the $z$-axes of the base frame $\{B\}$ and the tool frame $\{T\}$ are aligned with the axes of the first and last joints.
\begin{figure}[hptb]
        \centering
        \includegraphics[width=0.35\textwidth]{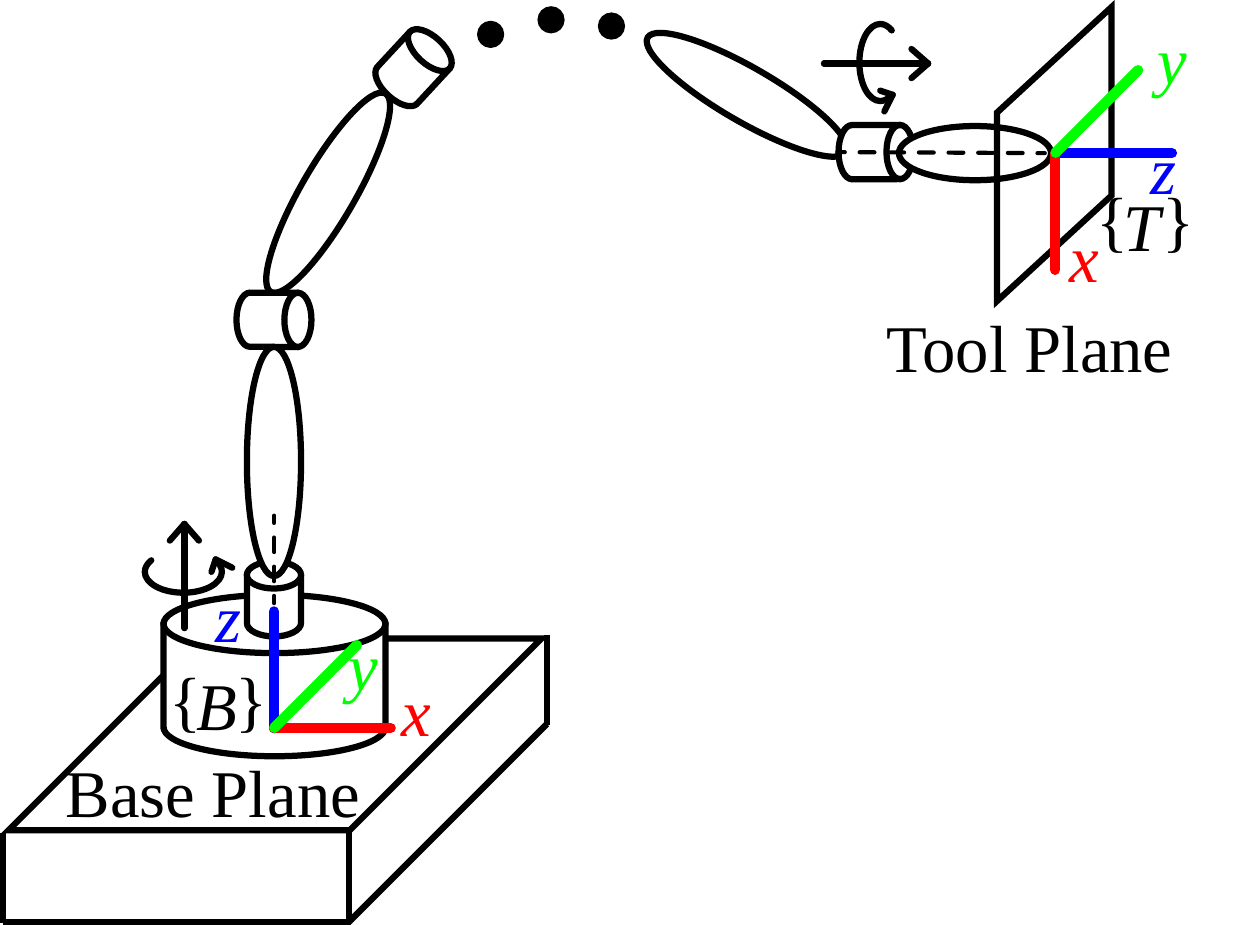}
        \caption{Applicable serial manipulator configuration for JAI of Jacobian-based performance indices.}
        \label{robot}
\end{figure}

In this paper, the considered Jacobian-based performance indices are defined under the assumption that the evaluated task Jacobian has full row rank, which implies $m \leq n$. This setting includes both nonredundant and redundant cases. For lower-dimensional tasks, the performance indices are evaluated only within the selected attainable task subspace.

\subsection{Properties and Proof of JAI}

JAI primarily focuses on the first and the last (or $n$-th) joints. An essential objective is to prove that the performance indices listed in Table~\ref{Jacobian-Based Performance Indices for Serial Manipulators} are independent of these two joint angles under specified conditions. A key advantage of JAI is that it enables both angles to be set to zero, thereby increasing matrix sparsity and improving computational efficiency.

Rotate the first and last joints by arbitrary angles, i.e., $\bm q_{\Delta}=\bm q + [q_{1\Delta},0,...,0,q_{n\Delta}]^{\rm T}$. Thus the pose transformation is:
\begin{equation}
\begin{gathered}
  {\bm R}_{\Delta}=\bm R_z(q_{1\Delta})\bm R\bm R_z(q_{n\Delta})  \hfill \\
  \bm p_{\Delta}=\bm R_z(q_{1\Delta})\bm p \hfill \\ 
\end{gathered} 
\end{equation}
$\bm R$ and $\bm p$ represent the robot's orientation and position (subscript ``$\Delta$'' denotes transformed values). $\bm R_z( \bullet )$ denotes the rotation matrix about the $z$-axis. 

The following derivation is carried out in the 6-dimensional Cartesian space ($m=6$). Lower-dimensional performance evaluation is defined on an $l$-dimensional task subspace. Therefore, the applicability of the subsequent JAI analysis depends on whether the selected task subspace is invariant under the transformation induced by the rotations of the first and last joints.

The body velocity undergoes the following transformation:
\begin{equation}
\begin{gathered}
\begin{aligned}
\bm \omega_{\Delta}
&=\left(\bm R_{\Delta}^{\rm T}\bm{\dot R}_{\Delta}\right)^{ \vee }=\left(\bm R_z(q_{n\Delta})^{\rm T}\bm R^{\rm T}\bm{\dot R}\bm R_z(q_{n\Delta})\right)^{\vee}\\
&=\left(\bm R_z(q_{n\Delta})^{\rm T}\bm \omega\bm R_z(q_{n\Delta})\right)^{\vee}=\bm R_z(q_{n\Delta})^{\rm T}\bm \omega
\end{aligned}   \hfill \\
\bm v_{\Delta}=\bm R_{\Delta}^{\rm T} \dot{\bm p}_{\Delta}=\bm R_z(q_{n\Delta})^{\rm T}\bm R^{\rm T}\dot{\bm p}=\bm R_z(q_{n\Delta})^{\rm T}\bm v \hfill \\ 
\end{gathered} 
\end{equation}
where $(\bullet)^{\vee}$ denotes the isomorphic transformation from a skew-symmetric matrix to a vector. Accordingly, the transformed body Jacobian is denoted as
\begin{equation}
\begin{aligned}
&\bm J_{\Delta}
=\bm A \bm J= \left[ {\begin{array}{*{20}{c}}
\bm R_z(q_{n\Delta})^{\rm T}&{\bm 0} \\ 
{\bm 0}&\bm R_z(q_{n\Delta})^{\rm T}
\end{array}} \right] \bm J \\
\end{aligned}
\label{A}
\end{equation}
where $\bm A$ denotes the adjoint operator for a pure $z$-axis rotation. We aim to discuss the transformation properties of the quadratic form matrices associated with both the manipulability and stiffness ellipsoids:
\begin{equation}
\begin{gathered}
\left( \bm J_{\Delta}\bm J_{\Delta}^{\rm T}\right)^{-1}=\bm A \left( \bm J\bm J^{\rm T}\right)^{-1}\bm A^{\rm T}  \\
\bm K_{\Delta}^{-2}
=\left(\bm A \bm J \bm K_q^{-1}\bm J^{\rm T}\bm A^{\rm T}\right)\left(\bm A \bm J \bm K_q^{-1}\bm J^{\rm T}\bm A^{\rm T}\right) =\bm A \bm K^{-2}\bm A^{\rm T}  \\ 
\end{gathered} 
\end{equation}
Clearly, the same transformation applies to both quadratic form matrices. By denoting both $(\bm{JJ})^{-1}$ and $\bm K^{-2}$ uniformly as $\bm B$ (general symmetric positive definite matrix), we now discuss the JAI for all performance indices listed in Table~\ref{Jacobian-Based Performance Indices for Serial Manipulators}.

1) $M_D,M_P,M_V,K_D,K_P$ and $K_V$

For the sake of notational uniformity, we denote $\bm \xi$, $\bm \zeta$, $\bm P$ and $\bm Q$ by $\bm U \in \mathbb{R}^{6 \times l}(1 \leqslant l \leqslant 6)$, whose column vectors are orthonormal and $\bm U^{\rm T}\bm U= \mathbf{ I}_{l \times l}$. Besides, these six performance indices share the same mathematical form: $1/\sqrt{\det(\bm{U}^{\rm T}\bm B\bm{U})}$. Accordingly, the condition of JAI for these performance indices can be formulated as follows:
\begin{equation}
\det\left(\bm{U}^{\rm T}\bm B\bm{U}\right)=\det\left(\bm{U}^{\rm T}\bm A\bm B \bm A^{\rm T}\bm{U}\right)
\label{invariance}
\end{equation}
We aim to identify some suitable $\bm U$ satisfying \eqref{invariance} to clarify the applicable conditions. The following theorem is given.
\begin{theorem}
Let $\bm A \in \mathbb R^{m \times m}$ be an orthogonal matrix and $\bm U\in \mathbb R^{m \times l} (m \geqslant l)$ be a matrix with orthonormal columns. For all symmetric positive definite matrices $\bm B \in \mathbb R^{m \times m}$, the necessary and sufficient condition for \eqref{invariance} to hold is that the column spaces of $\bm U$ and $\bm A^{\rm T}\bm U$ are identical: $ {\rm Col}\left(\bm U\right)= {\rm Col}\left(\bm A^{\rm T}\bm U\right)$.
\end{theorem}
\begin{proof}
(i) Sufficiency: Given that $ {\rm Col}\left(\bm U\right)= {\rm Col}\left(\bm A^{\rm T}\bm U\right)$, every column of $\bm A^{\rm T}\bm U$ lies in the column space of $\bm U$, which implies that there exists a matrix $\bm C \in \mathbb{R}^{l \times l}$ such that $\bm A^{\rm T}\bm U = \bm U\bm C$. Since $(\bm A^{\rm T}\bm U)^{\rm T}\bm A^{\rm T}\bm U=\bm U^{\rm T}\bm A\bm A^{\rm T}\bm U=\mathbf{ I}_{l \times l}$, we have $(\bm U\bm C)^{\rm T}\bm U\bm C=\bm C^{\rm T}\bm C=\mathbf{ I}_{l \times l}$. Thus, $\bm C $ is an orthogonal matrix satisfying $\left(\det\left(\bm C\right)\right)^2=1$. These imply that 
$\det\left(\bm{U}^{\rm T}\bm A\bm B \bm A^{\rm T}\bm{U}\right)
=\det\left((\bm U\bm C)^{\rm T}\bm B (\bm U\bm C)\right) 
=\left(\det\left(\bm C\right)\right)^2\det\left(\bm{U}^{\rm T}\bm B\bm{U}\right)=\det\left(\bm{U}^{\rm T}\bm B\bm{U}\right)$, which completes the proof of sufficiency.

(ii) Necessity: We prove necessity by contradiction. Assume that $ {\rm Col}\left(\bm U\right) \ne  {\rm Col}\left(\bm A^{\rm T}\bm U\right)$. Consider a symmetric positive definite matrix $\bm B$ of the form: $\bm B=\mathbf{I}_{m\times m}+\alpha \bm b \bm b^{\rm T}$, where $\alpha>0$ and $\bm b \in \mathbb{R}^{m\times 1}$ is an arbitrary vector. Based on the matrix determinant lemma~\cite{ding2007eigenvalues}, we can obtain
\begin{equation}
\begin{gathered}
\begin{aligned}
\det\left(\bm{U}^{\rm T}\bm B\bm{U}\right)
&=\det\left(\mathbf{I}_{l\times l}+\alpha\bm U^{\rm T}\bm b(\bm U^{\rm T}\bm b)^{\rm T}\right) \\
&=1+\alpha \bm b ^{\rm T} \bm U \bm U^{\rm T}\bm b
\end{aligned}
\hfill \\
\begin{aligned}
\det\left(\bm{U}^{\rm T}\bm A\bm B \bm A^{\rm T}\bm{U}\right)
&=\det\left(\mathbf{I}_{l\times l}+\alpha\bm U^{\rm T} \bm A \bm b(\bm U^{\rm T} \bm A \bm b)^{\rm T}\right) \\
&=1+\alpha \bm b^{\rm T} (\bm A^{\rm T}\bm U)(\bm A^{\rm T}\bm U)^{\rm T}\bm b
\end{aligned}
\hfill \\
\end{gathered} 
\end{equation}
The two expressions above are equal for arbitrary $\bm b$ if and only if the quadratic form matrices coincide, i.e., $\bm U \bm U^{\rm T}=(\bm A^{\rm T}\bm U)(\bm A^{\rm T}\bm U)^{\rm T}$, which correspond exactly to the orthogonal projection matrices onto $\bm U$ and $\bm A^{\rm T} \bm U$. Specifically, $\bm U(\bm U^{\rm T}\bm U)^{-1}\bm U^{\rm T}=\bm U \bm U^{\rm T}$ and $\bm A^{\rm T} \bm U((\bm A^{\rm T} \bm U)^{\rm T}\bm A^{\rm T} \bm U)^{-1}(\bm A^{\rm T} \bm U)^{\rm T}=(\bm A^{\rm T}\bm U)(\bm A^{\rm T}\bm U)^{\rm T}$ hold. The equality of the orthogonal projection matrices implies that the column spaces of the two matrices are identical, i.e., ${\rm Col}(\bm U)= {\rm Col}(\bm A^{\rm T}\bm U)$, contradicting the initial assumption. Hence, the necessity of the theorem holds.
\end{proof}

The JAI of these six performance indices requires that ${\rm Col}(\bm U)= {\rm Col}(\bm A^{\rm T}\bm U)$, i.e., the column space of $\bm U$ forms an $\bm A^{\rm T}$-invariant subspace. In 3-dimensional space, the rotation operator $\bm R_z( \bullet )$ admits two invariant proper subspaces: (i) $[x,y,0]^{\rm T}$ lies within the $xy$-plane after rotation about the $z$-axis; (ii) $[0,0,z]^{\rm T}$ remains unaffected. In 6-dimensional full Cartesian space ($m=6$), depending on the arbitrariness of $q_{n\Delta}$ and the block diagonal form of $\bm A^{\rm T}$ (from \eqref{A}), $\bm U$ admits the following three representations:
\begin{equation}
\bm U = \bm U_{xy}, \, \bm U = \bm U_{z} \ {\rm or} \  \bm U = \left[ \bm U_{xy}, \, \bm U_{z} \right]
\label{U}
\end{equation}
Define $[\bm e_{x1}, \bm e_{y1}, \bm e_{z1}, \bm e_{x2}, \bm e_{y2}, \bm e_{z2}]=\mathbf{I}_{6\times 6}$ as the standard basis of 6-dimensional space. Accordingly, 
\begin{equation}
\begin{gathered}
\bm U_{xy}  \subseteq {\rm Span}\{\bm e_{x1}, \bm e_{y1},\bm e_{x2}, \bm e_{y2}\} \\   
\bm U_z \subseteq {\rm Span}\{ \bm e_{z1}, \bm e_{z2} \}
\end{gathered}
\end{equation}
where ${\rm Span}\{\bullet\}$ denotes the vector space spanned by basis vectors. It can be shown that the above two subspaces are mutually orthogonal. Guided by the requirement of $\bm A^{\rm T}$-invariant subspaces, representative admissible components of $\bm U_{xy}$ and $\bm U_z$, whose columns are orthonormal, are listed in Table~\ref{invariant subspaces}. These basis sets can be used to construct admissible task subspaces for which the JAI property is valid. Within these subspaces, the property of JAI is satisfied, as the six performance indices ($M_D,M_P,M_V,K_D,K_P$ and $K_V$) are unaffected by variations in $q_{1\Delta}$ and $q_{n\Delta}$.

It should be emphasized that Table~\ref{invariant subspaces} does not imply that an arbitrary lower-dimensional task space automatically satisfies the JAI property. The corresponding projected twist and projected Jacobian are given by $\bm V_U=\bm U^{\rm T}\bm V$ and $\bm J_U=\bm U^{\rm T}\bm J$. The directional, planar, or volumetric performance indices are then evaluated within this task subspace. The JAI property holds only when the selected subspace satisfies $ {\rm Col}\left(\bm U\right)= {\rm Col}\left(\bm A^{\rm T}\bm U\right)$.

\begin{table}[htbp]
        \centering
        \renewcommand{\arraystretch}{1.5}
        \caption{Representative $\bm A^{\rm T}$-invariant subspaces spanned by orthonormal basis vectors.}
        \label{invariant subspaces}
        \begin{tabular}{m{0.35cm}<{\centering}m{1.1cm}<{\centering}m{5.7cm}<{\centering}}
        \toprule
         & Dimension & Mathematical Form  \\
        \midrule
        \multirow{6}{*}{$\bm U_{xy}$} & (i) $l=2$ & $[\cos\psi \, \bm e_{x1}+\sin\psi \, \bm e_{y1},-\sin\psi \, \bm e_{x1}+\cos\psi \, \bm e_{y1}]$,  $\psi \in [0,2\pi)$ \\
        & (ii) $l=2$ & $[\cos\psi \, \bm e_{x2}+\sin\psi \, \bm e_{y2},-\sin\psi \, \bm e_{x2}+\cos\psi \, \bm e_{y2}]$,  $\psi \in [0,2\pi)$ \\
        &  \multirow{1.3}{*}{(iii) $l=2$} & $\dfrac{1}{\sqrt{1+\psi^2}} \, [\bm e_{x1}+\psi\bm e_{x2},\bm e_{y1}+\psi\bm e_{y2}]$,  $\psi \ne 0$ \\
        &  \multirow{1.9}{*}{(iv) $l=4$} & $\begin{gathered}
        \left[\cos\psi \, \bm e_{x1}+\sin\psi \, \bm e_{y1},-\sin\psi \, \bm e_{x1}+\cos\psi \, \bm e_{y1}, \right.  \\[-0.12cm]
         \left. \cos\varphi \, \bm e_{x2}+\sin\varphi \, \bm e_{y2},-\sin\varphi \, \bm e_{x2}+\cos\varphi \, \bm e_{y2}\right],  \\[-0.12cm]
          \psi,\varphi \in [0,2\pi)  \\ 
        \end{gathered} $ \\
        \midrule
        \multirow{2}{*}{$\bm U_{z}$} & (i) $l=1$ & $\cos\psi \, \bm e_{z1}+\sin\psi \, \bm e_{z2}, \ \psi \in [0,2\pi)$ \\
        & (ii) $l=2$ & $[\bm e_{z1},\bm e_{z2}]$ \\
        \bottomrule
        \end{tabular}
\end{table}

2) $M_{C,2}$, $M_{C,F}$, $K_{C,2}$ and $K_{C,F}$

The following conditions hold:

(i) $\bm J_{\Delta}\bm J_{\Delta}^{\rm T}=\bm A \bm J(\bm A\bm J)^{\rm T}=\bm A(\bm J \bm J^{\rm T}) \bm A^{\rm T}$ and $(\bm J_{\Delta}\bm J_{\Delta}^{\rm T})^{-1}=\bm A(\bm J \bm J^{\rm T})^{-1} \bm A^{\rm T}$;

(ii) $\bm K_{\Delta}=(\bm A \bm J \bm K_q^{-1}(\bm A\bm J)^{\rm T})^{-1}=\bm A \bm K \bm A^{\rm T}$, $\bm K_{\Delta}^2=\bm A \bm K^2 \bm A^{\rm T}$ and $\bm K_{\Delta}^{-2}=\bm A \bm K^{-2} \bm A^{\rm T}$.

The four indices undergo the identical orthogonal similarity transformation, which preserves the eigenvalues. Furthermore, the singular values of $\bm J$ are the square roots of the eigenvalues of $\bm J\bm J^{\rm T}$. Therefore, $M_{C,2}$, $M_{C,F}$, $K_{C,2}$ and $K_{C,F}$ are independent of $q_{1\Delta}$ and $q_{n\Delta}$, thus exhibiting JAI.

3) Geometric representation

A geometric representation of JAI in 3-dimensional Cartesian space ($m=3$) is provided, as illustrated in Fig.~\ref{Geometric representation of the joint-angle}. When the first and last joints rotate, the quadratic form matrix of the manipulability/stiffness ellipsoid undergoes an orthogonal similarity transformation, which manifests geometrically as a rotation about the $z$-axis. Under this condition, the JAI of the directional ($M_D$ and $K_D$), planar ($M_P$ and $K_P$), and volumetric ($M_V$ and $K_V$) performance indices is maintained when (for the simplest case $\psi=0$)
\begin{equation}
\bm \xi =\bm \zeta=\left[ {\begin{array}{*{20}{c}}
  0 \\ 
  0 \\ 
  1 
\end{array}} \right], \ \bm P= \left[ {\begin{array}{*{20}{c}}
  1 & 0 \\ 
  0& 1 \\ 
  0&0 
\end{array}} \right], \ \bm Q=\mathbf{I}_{3\times 3},
\end{equation}
respectively. Furthermore, the semi-axis lengths, which are determined by the eigenvalues of the quadratic-form matrix, remain constant. Thus, $M_{C,2}$, $M_{C,F}$, $K_{C,2}$ and $K_{C,F}$ are invariant.

\begin{figure}[hptb]
        \centering
        \includegraphics[width=0.485\textwidth]{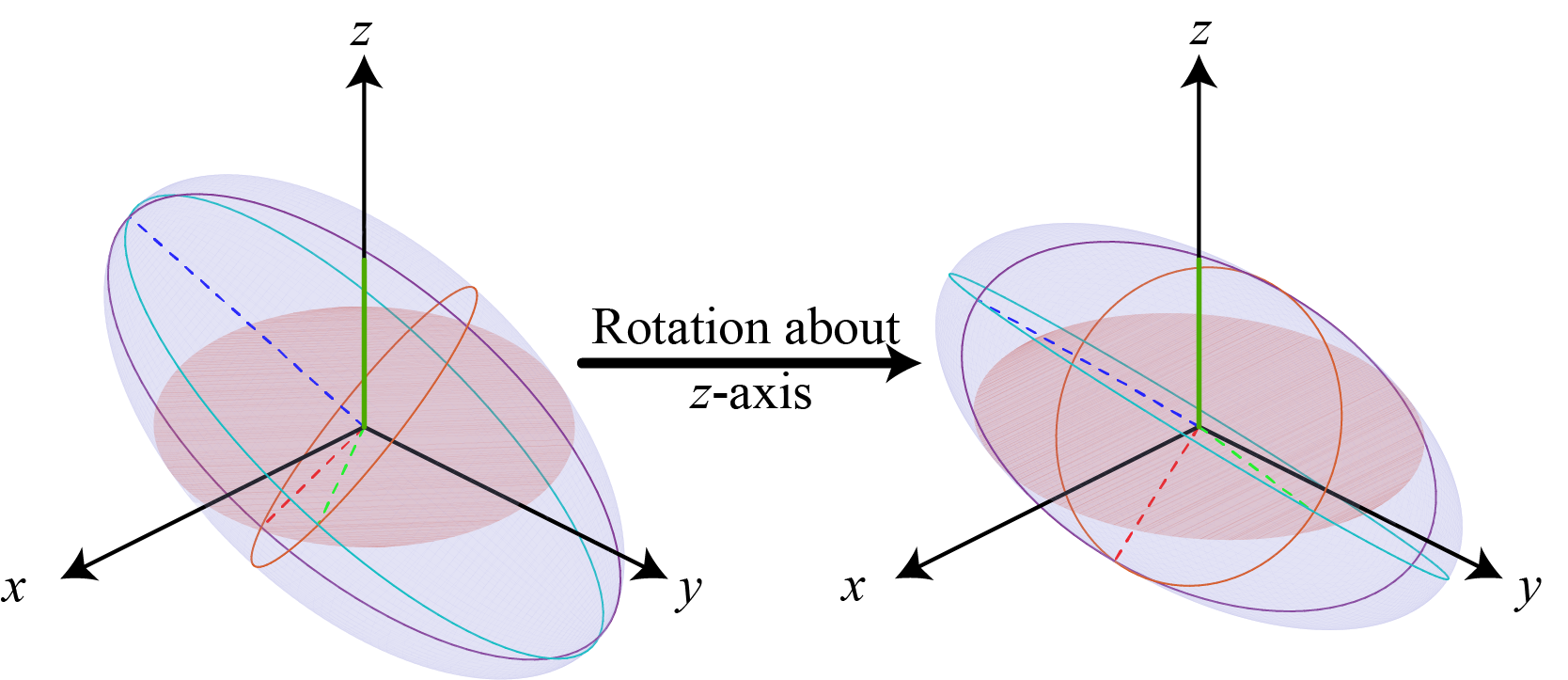}
        \caption{Geometric representation of JAI in 3-dimensional Cartesian space ($m=3$). When the first and last joints rotate, due to $(\bm J_{\Delta}\bm J_{\Delta}^{\rm T})^{-1}=\bm A(\bm J \bm J^{\rm T})^{-1} \bm A^{\rm T}$ and $\bm K_{\Delta}^{-2}=\bm A \bm K^{-2} \bm A^{\rm T}$, the manipulability/stiffness ellipsoid exhibits a rotation about the $z$-axis. The invariants comprise four quantities: (i) the intercept of the ellipsoid on the $z$-axis (deep green line segment); (ii) the cross-sectional area of the ellipsoid within the $xy$-plane (dusty pink plane); (iii) the volume of the ellipsoid (light purple ellipsoid); (iv) the lengths of all semi-axes of the ellipsoid (three dashed line segments in red, green, and blue).}
        \label{Geometric representation of the joint-angle}
\end{figure}

\section{Application Scenarios}

This section provides two application scenarios: performance map construction and task layout optimization.

\subsection{Performance map construction}
\label{Performance map construction}

A robot performance map describes how performance indices are distributed across the reachable workspace. Studying performance maps is important for robot design, analysis, task planning, and optimization. In this work, we adopt a joint-space method, where joint configurations are sampled within their limits and the corresponding performance values are obtained through forward kinematics.

This subsection presents the construction of performance maps for the 7-DOF Franka Research 3 robot, whose mechanism configuration satisfies the proposed JAI. The relevant robot parameters can be obtained from the \href{https://frankarobotics.github.io/docs/}{official website}. Among the Jacobian-based performance indices in Table~\ref{Jacobian-Based Performance Indices for Serial Manipulators}, two widely used indices are considered as case studies: overall manipulability index $\mathcal{M}=\sqrt{\det(\bm {JJ}^{\rm T})}$ and overall stiffness index $\mathcal{K}=\sqrt[6]{\det(\bm K)}$, corresponding to $M_V$ and $K_V$ (normalized) when $\bm Q=\mathbf{I}_{6\times 6}$. The joint stiffness ${\bm K}_q$ is assumed to be ${\rm diag}(3,3,2,2,2,1,1)\times 10^4 $ Nm/rad. Two strategies for constructing performance maps are compared:

\begin{itemize}
\item Conventional strategy (CS): Sample the full joint space to construct performance maps.
\item JAI-based strategy (JAIBS): When constructing the performance map, avoid sampling $q_1$ and $q_7$. Subsequently, considering the influence of $q_1$ on the robot's position, rotate and replicate this map about the $z$-axis of the base frame, and the complete performance map can be obtained. The rotation of the last joint affects neither the robot's position nor the aforementioned performance indices. When calculating the performance index, set $q_1=q_7=0$ to enhance matrix sparsity.
\end{itemize}

For sampling, the physical range of each joint is uniformly discretized into 10 sampling points. Table~\ref{Comparative Results of Computational Efficiency and Storage Space between CS and JAIBS} compares the two strategies in terms of computational efficiency and storage requirements. The results demonstrate the advantages of JAIBS: (i) The computational load is reduced to 1/100 of that required by CS, with the corresponding CPU time decreasing to 1.723 s; (ii) As the data size decreases to 1/10 of its original, the storage space required for the ``.mat'' file decreases by 48.6\%. Fig.~\ref{Performance maps of the Franka Research 3} presents the relative performance difference maps obtained by both strategies. The relative differences of $\mathcal{M}$ and $\mathcal{K}$ remain at the levels of $10^{-10}$ and $10^{-11}$, indicating that the two strategies produce almost identical performance distributions. The observed discrepancies mainly originate from numerical approximation, rather than from a violation of JAI. These results confirm that JAIBS effectively mitigates the curse of dimensionality in high-dimensional computations while preserving the accuracy of the performance maps.

\begin{table}[htbp]
        \centering
        \renewcommand{\arraystretch}{1.5}
        \caption{Comparative results of computational efficiency and storage space between CS and JAIBS.}
        \label{Comparative Results of Computational Efficiency and Storage Space between CS and JAIBS}
        \begin{tabular}{m{0.7cm}<{\centering}m{2.5cm}<{\centering}m{0.9cm}<{\centering}m{1.8cm}<{\centering}m{0.7cm}<{\centering}}
        \toprule
        Strategy & Computational load & CPU time (s) & Data size & Storage space (KB) \\
        \midrule
        CS & $10^{7} \times \text{Sampling}$ & 163.887 & $10^{7} \times \text{Sampling}$ & 35682 \\
        JAIBS & $10^{5} \times \text{Sampling} + \text{Rotation} + \text{Replication}$ & 1.723 & $10^{6} \times \text{Sampling}$ & 18329 \\
        \bottomrule
        \end{tabular}
\end{table}

\begin{figure}[hptb]
        \centering
        \subfigure[]{
        \includegraphics[width=0.23\textwidth]{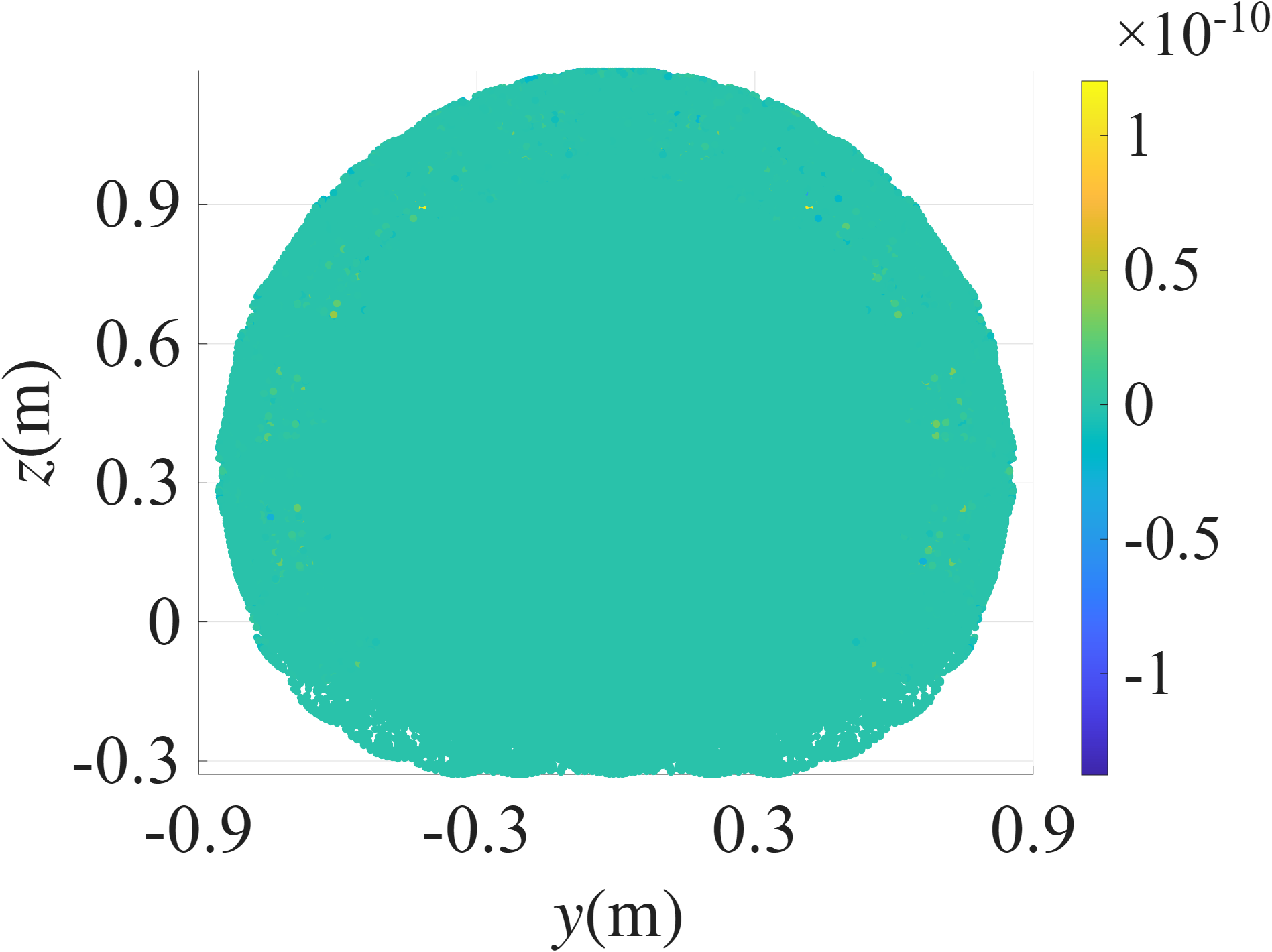}}
        \subfigure[]{
        \includegraphics[width=0.23\textwidth]{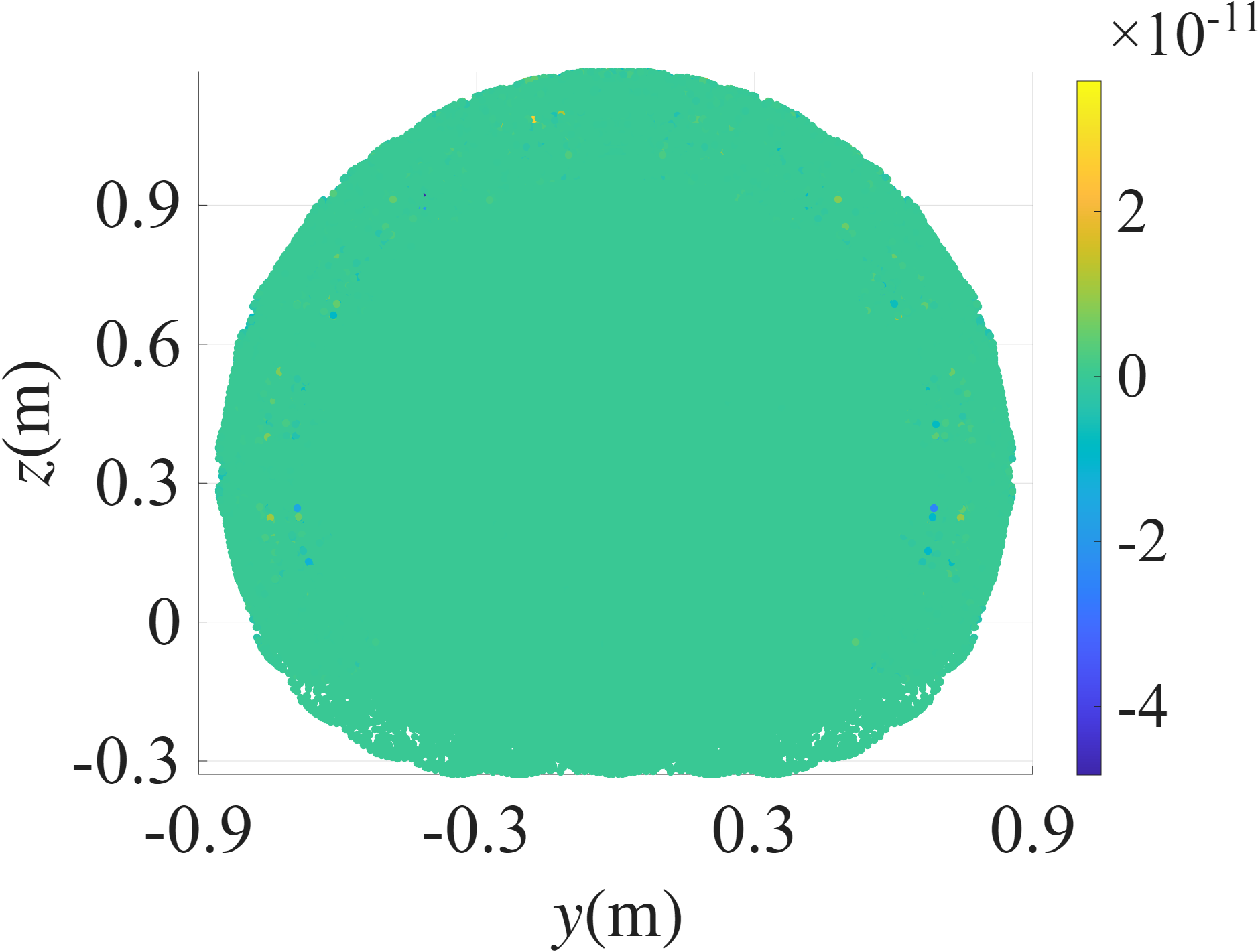}}
        \caption{Pointwise relative performance difference maps between CS and JAIBS. (a) $\mathcal{M}$. (b) $\mathcal{K}$.}
        \label{Performance maps of the Franka Research 3}
\end{figure}

\subsection{Task layout optimization}

Within the reachable workspace, a robot can typically accomplish the same task in different operational regions, as illustrated in Fig.~\ref{Different task layouts on a horizontal plane}. The frame $\{W\}$, which defines the task layout, is highly correlated with the robot's operational performance. For instance, performing a polishing task within a high-stiffness region can effectively mitigate vibrations. Conversely, executing lightweight operations (e.g., plotting or spraying) in a high-manipulability region improves motion dexterity and trajectory tracking accuracy. Consequently, task layout optimization is a crucial step for enhancing task quality. In this work, the task layout optimization for $\{W\}$ is constrained to a horizontal plane, with 3-DOF optimization variables: translations along the $x$-axis and $y$-axis, and rotation about the $z$-axis, denoted collectively as $\bm p_w = [x_w,y_w,\varphi_w]$.

\begin{figure}[hptb]
        \centering
        \includegraphics[width=0.4\textwidth]{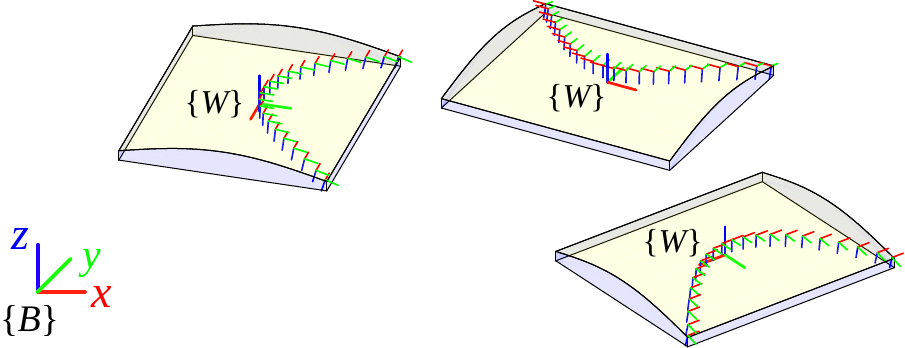}
        \caption{Different task layouts on a horizontal plane.}
        \label{Different task layouts on a horizontal plane}
\end{figure}

This work focuses on the performance indices satisfying JAI, whose distribution on the horizontal plane exhibits a ring-shaped pattern (e.g., Fig.~7 in \cite{zhou2024direct}), as it does not depend on the first joint of the robot. Leveraging this property, an equivalent dimensionality reduction for the task layout optimization can be achieved. Specifically, by rotating $\{W\}$ about the $z$-axis of the base frame (i.e., the first joint axis) until it aligns with the positive $x$-axis, the task layout can be uniquely represented by a reduced-dimensional pose 
\begin{equation}
\bm p_{w}'=[x_w \cos \varphi_{w\Delta}-y_w \sin \varphi_{w\Delta}, \, 0 \, , \, \varphi_w+\varphi_{w\Delta}]
\label{Equivalent task layout}
\end{equation}
where the rotation angle $\varphi_{w\Delta}=-{\rm atan2}(y_w,x_w)$. By correspondingly adjusting the trajectory of the first joint ($q_{1\Delta}=\varphi_{w\Delta}$), the robot can still accomplish the operational task without altering its trajectory performance. This implies that the 2-dimensional optimization variable $[x_w,0,\varphi_w]$ is theoretically equivalent to the original 3-dimensional one in terms of effectiveness. Such dimensionality reduction of the optimization variable can significantly improve computational efficiency. In the task layout optimization (different from Section~\ref{Performance map construction}):
\begin{itemize}
        \item CS: Employ $[x_w,y_w,\varphi_w]$ as the optimization variable.
        \item JAIBS: Employ $[x_w,\varphi_w]$ ($y_w \equiv 0$) as the optimization variable.
\end{itemize}

Two application scenarios are designed to demonstrate the practical utility of JAI, as illustrated in Fig.~\ref{Application scenarios for task layout optimization}. 

(i) Task A: For planar plotting, the Franka Research 3 robot is employed to trace an Archimedean spiral while maintaining a constant orientation. This task requires high flexibility and precision while avoiding singular configurations. Since the motion is restricted to translations within the $xy$-plane, we set the optimization objective as the planar manipulability index: $M_P=1/\sqrt{\det(\bm{P}^{\rm T}(\bm{JJ}^{\rm T})^{-1}\bm{P})}$, where $\bm P=[\mathbf{0}_{2\times 3},\mathbf{I}_{2\times 2},\mathbf{0}_{2\times 1}]^{\rm T}$. Optimizing this performance index enhances the robot's motion capability in the $xy$-plane. The inverse kinematics adopts a null-space gradient projection method:
\begin{equation}
\dot{\bm q}=\bm J^\#\bm V+k_0(\mathbf{I}_{7\times 7}-\bm J^\#\bm J)(\partial M_P/\partial \bm q)
\label{gradient projection method}
\end{equation}
where $k_0=1$ and $\bm J^\#=\bm J^{\rm T}(\bm J\bm J^{\rm T})^{-1}$.

(ii) Task B: Curved surface polishing using a 6-DOF IRB 4400 robot. The desired polishing path lies on an arc surface, whose projection onto the horizontal plane forms a cosine curve. The $z$-axis of the robot's tool frame always remains aligned with the surface normal. The polishing task is generally characterized as a low-speed process. Keeping robots in high-stiffness regions can improve force control stability, thereby enhancing polishing quality~\cite{wang2026posture}. Since the polishing force is applied along the $z$-axis of the tool frame, the optimization objective focuses on the directional stiffness index: $K_D=1/\sqrt{\bm{\zeta}^{\rm T}\bm{K}^{-2}\bm{\zeta}}$, where $\bm{\zeta}=[0,0,0,0,0,1]^{\rm T}$. The inverse kinematics of the robot is solved using a closed-form solution based on Pieper's Criterion. 

\begin{figure}[hptb]
        \centering
        \subfigure[]{
        \includegraphics[width=0.2\textwidth]{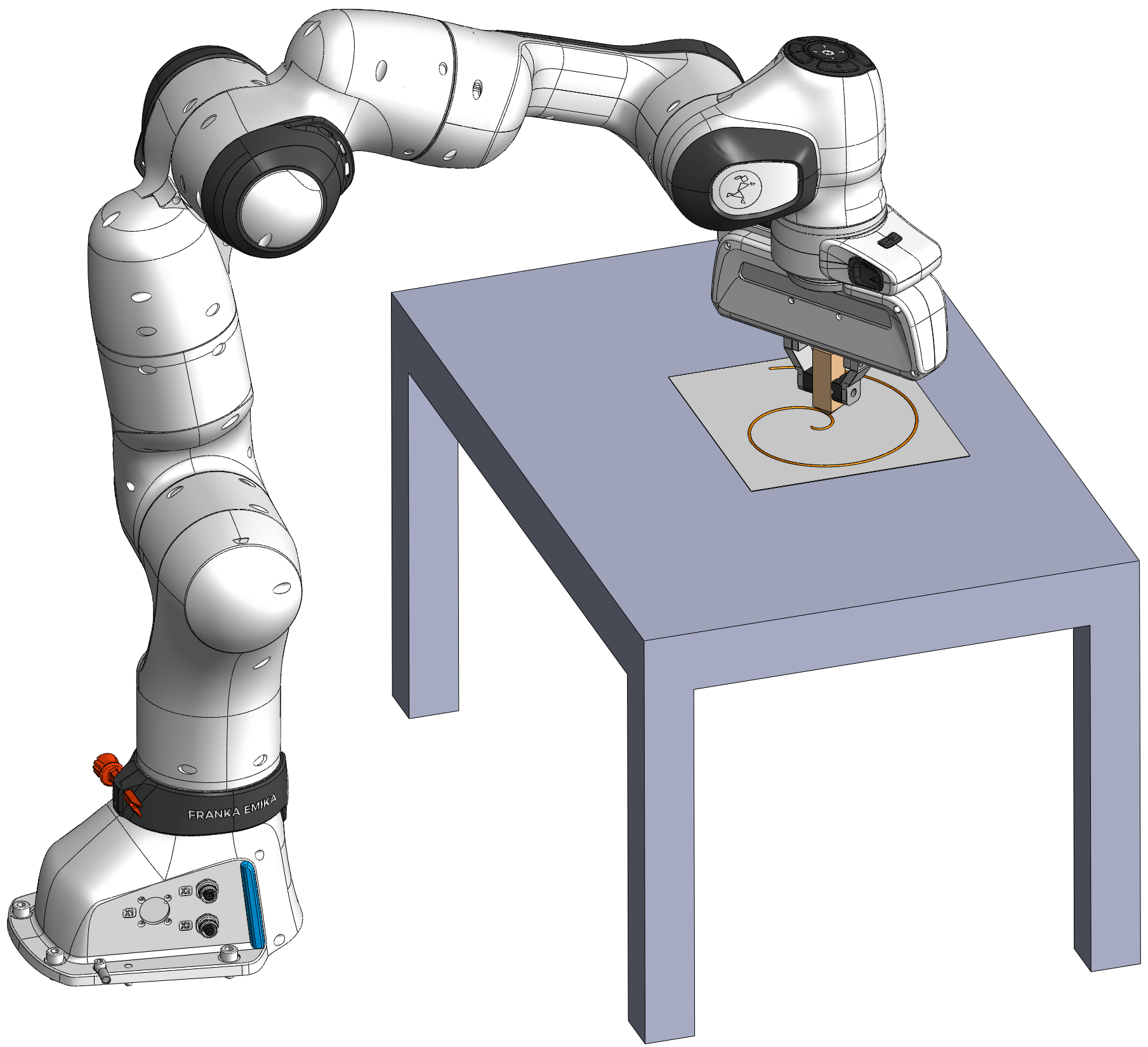}}
        \hspace{0.1cm}
        \subfigure[]{
        \includegraphics[width=0.21\textwidth]{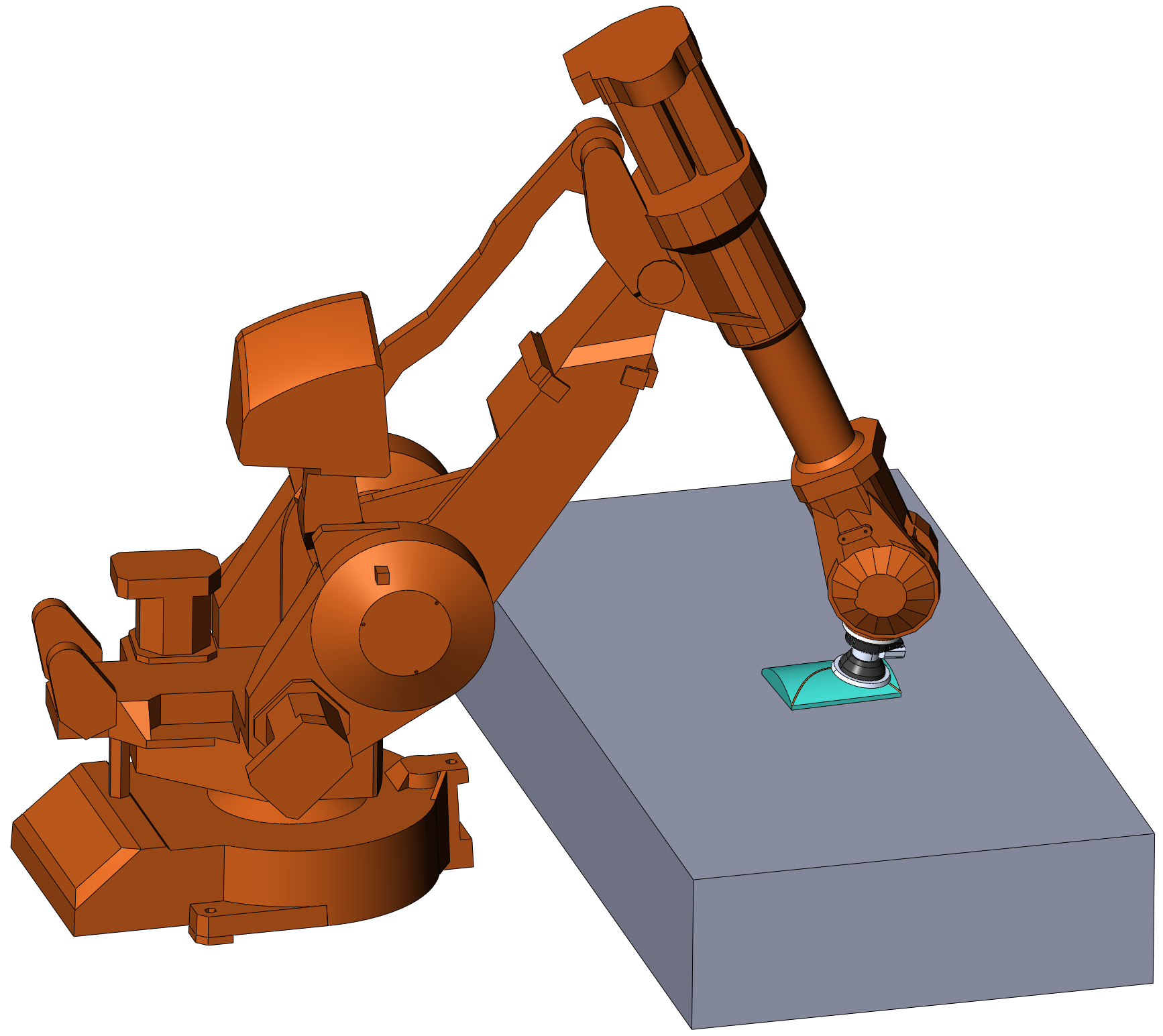}}
        \caption{Application scenarios for task layout optimization. (a) Task A: plotting using Franka Research 3. (b) Task B: polishing using IRB 4400.}
        \label{Application scenarios for task layout optimization}
\end{figure}

The total durations for Task A and Task B are 5 s and 15 s. Trajectory planning is conducted within a sequential solution framework with a sampling period of 0.01 s. The task layout optimization problem is solved in MATLAB using the interior-point method. The step tolerance is set to $10^{-5}$. The two serial robots conform to the configuration requirements illustrated in Fig.~\ref{robot}.

Table~\ref{Task layout optimization results of CS and JAIBS} presents the results of task layout optimization. Each case includes three outcomes derived from random initial conditions. The ``Average performance'' represents the mean performance index during task execution ($M_P$ for Task A and $K_D$ for Task B), and the ``Equivalent task layout'' is obtained from~\eqref{Equivalent task layout}. For JAIBS, this is equivalent to the task layout itself.

In terms of optimization effectiveness, both strategies perform comparably, with differences remaining within 0.1\%. Due to the low dimensionality of the optimization variables, both strategies exhibit favorable convergence. Nevertheless, the task layout results obtained by CS are sensitive to initial conditions and computational procedures, whereas JAIBS yields consistent task layouts, thereby demonstrating superior robustness and consistency. Moreover, the equivalent task layouts derived from CS are nearly identical to those of JAIBS, with only minor deviations in Task A attributable to numerical iterative computations of~\eqref{gradient projection method}. In this respect, Task B, which employs closed-form inverse kinematics, shows smaller deviations and stronger computational stability than Task A, which relies on numerical solutions. 

As discussed above, redundancy exists in the optimization variables due to JAI (primarily concerning the first joint): a 3-dimensional task layout effectively has only 2 DOFs. Consequently, JAIBS incurs substantially lower computational costs, reducing them by approximately 45.7\% to 73.7\% compared with CS. Although Task B employs a closed-form solution, its computational efficiency is not significantly higher than Task A. This can be attributed to two factors: (i) the number of solutions required for Task B is about three times that of Task A (15 s vs. 5 s total duration); (ii) based on the sequential solution strategy framework, the number of iterations of~\eqref{gradient projection method} is very small, typically fewer than 5. In summary, for task layout optimization, JAIBS significantly improves computational efficiency and stability while maintaining performance, making it suitable for practical applications.

\begin{table*}[htbp]
\centering
\renewcommand{\arraystretch}{1.5}
\caption{Task layout optimization results of CS and JAIBS.}
\label{Task layout optimization results of CS and JAIBS}
\begin{tabular}{m{1cm}<{\centering}m{1cm}<{\centering}m{0.5cm}<{\centering}m{1.8cm}<{\centering}m{3.5cm}<{\centering}m{2.8cm}<{\centering}m{1.2cm}<{\centering}m{1.8cm}<{\centering}}
\toprule
Task & Strategy & No. & Average performance & Task layout (m,m,rad) & Equivalent task layout (m,m,rad) & CPU Time (s) & Average CPU time (s) \\
\midrule
\multirow{6.2}{*}{Task A} & \multirow{3.2}{*}{CS} & $1$ & $0.0967$ & $[0.4988,0.0503,1.8742]$ & $[0.5014,0,1.7737]$ & $14.794$ & \multirow{3.1}{*}{$17.604$} \\
&  & $2$ & $0.0967$ & $[0.4949,0.0813,1.9370]$ & $[0.5015,0,1.7742]$ & $23.367$ &  \\
&  & $3$ & $0.0967$ & $[0.5012,-0.0249,1.7236]$ & $[0.5018,0,1.7732]$ & $14.650$  &  \\
& \multirow{3.2}{*}{JAIBS} & $1$ & $0.0967$ & $[0.5015,0,1.7830]$ & $[0.5015,0,1.7830]$ & $6.181$  & \multirow{3.1}{*}{$6.231$} \\
&  & $2$ & $0.0967$ & $[0.5015,0,1.7828]$ & $[0.5015,0,1.7828]$ & $6.364$ 
 &   \\
&  & $3$ & $0.0967$ & $[0.5015,0,1.7828]$ & $[0.5015,0,1.7828]$ & $6.150$  &   \\
\midrule
\multirow{6.2}{*}{Task B} & \multirow{3.2}{*}{CS} & $1$ & $1.7768\times 10^6$ & $[-0.2250,0.7748,0.2826]$ & $[0.8068,0,-1.5708]$ &$ 14.219 $ & \multirow{3.1}{*}{$14.686$} \\
 &  & $2$ & $1.7768\times 10^6$ & $[0.5086,-0.6263,-2.4596]$ & $[0.8068,0,-1.5708]$ & $11.794 $ & \\
 &  & $3$ & $1.7768\times 10^6$ & $[0.7131,0.3773,-1.0842]$ & $[0.8068,0,-1.5708]$ & $18.044 $  &  \\
 & \multirow{3.2}{*}{JAIBS} & $1$ & $1.7768\times 10^6$ & $[0.8068,0,-1.5708]$ & $[0.8068,0,-1.5708]$ & $6.410$ & \multirow{3.1}{*}{$5.868$} \\
 &  & $2$ & $1.7768\times 10^6$ & $[0.8068,0,-1.5708]$ & $[0.8068,0,-1.5708]$ &$ 5.866$ & \\
 &  & $3$ & $1.7768\times 10^6$ & $[0.8068,0,-1.5708]$ & $[0.8068,0,-1.5708]$ & $5.327$ &  \\
\bottomrule
        \end{tabular}
\end{table*}

\section{Conclusions}

This study introduces the concept of JAI for Jacobian-based performance indices of serial robots. JAI is effective for specific indices, manifested as being independent of the first and last joints of the robot, thereby improving the efficiency of performance analysis and optimization. Through applications in performance map construction and task layout optimization, it is demonstrated that JAIBS achieves substantial gains in computational efficiency while preserving effectiveness. The results emphasize the potential of JAI to enhance the design, optimization, and planning for serial robots, particularly in scenarios requiring high computational efficiency.


\bibliographystyle{IEEEtran} 
\bibliography{IEEEabrv,ref}                        

\end{document}